\documentclass[a4paper,10pt]{article}
\pdfoutput=1

\usepackage[utf8]{inputenc}
\usepackage[T1]{fontenc}
\usepackage{lmodern}
\usepackage{amsmath}
\usepackage{amssymb}
\usepackage{amsfonts}
\usepackage{amsthm}
\usepackage{algorithm}
\usepackage{algorithmic}
\usepackage[font={footnotesize,sf},labelfont={footnotesize,sf,bf},skip=2.5pt]{caption}
\usepackage{graphicx}
\usepackage{color}
\usepackage{hyperref}
\usepackage{parskip}
\usepackage{complexity}
\usepackage{xspace}
\usepackage{multirow}
\usepackage[capitalize,nameinlink]{cleveref}
\usepackage{tikz}
\usetikzlibrary{positioning,shapes}
\usepackage{relsize}
\usepackage[labelformat=simple]{subcaption}

\usepackage{natbib}
\hypersetup{colorlinks,urlcolor=blue,linkcolor=blue,citecolor=blue,anchorcolor=blue}

\begingroup
    \makeatletter
    \@for\theoremstyle:=definition,remark,plain\do{%
        \expandafter\g@addto@macro\csname th@\theoremstyle\endcsname{%
            \addtolength\thm@preskip\parskip
            }%
        }
\endgroup

\newclass{\DEBATE}{DEBATE}
\newclass{\QBF}{QBF}

\newtheorem*{claim}{Claim}

\newtheorem{theorem}{Theorem}
\theoremstyle{definition}
\theoremstyle{theorem}

\newcommand{\catname}[1]{{\normalfont\textbf{#1}}}
\newcommand{\Set}{\catname{Set}\xspace}
\newcommand{\Group}{\catname{Group}\xspace}

\title{AI safety via debate}
\author{Geoffrey Irving\thanks{Corresponding author: \nolinkurl{irving@openai.com}}
    \and Paul Christiano \\ \\ OpenAI \and Dario Amodei}
\date{}

\begin{document}

\maketitle

\begin{abstract}
To make AI systems broadly useful for challenging real-world tasks, we need them to learn complex human goals and preferences.  One approach to specifying complex goals asks humans to judge during training which agent behaviors are safe and useful, but this approach can fail if the task is too complicated for a human to directly judge.  To help address this concern, we propose training agents via self play on a zero sum \emph{debate} game.  Given a question or proposed action, two agents take turns making short statements up to a limit, then a human judges which of the agents gave the most true, useful information.  In an analogy to complexity theory, debate with optimal play can answer any question in $\PSPACE$ given polynomial time judges (direct judging answers only $\NP$ questions).  In practice, whether debate works involves empirical questions about humans and the tasks we want AIs to perform, plus theoretical questions about the meaning of AI alignment.  We report results on an initial MNIST experiment where agents compete to convince a sparse classifier, boosting the classifier's accuracy from 59.4\% to 88.9\% given 6 pixels and from 48.2\% to 85.2\% given 4 pixels.  Finally, we discuss theoretical and practical aspects of the debate model, focusing on potential weaknesses as the model scales up, and we propose future human and computer experiments to test these properties.
\end{abstract}

\section{Introduction} \label{sec:intro}

Learning to align an agent's actions with the values and preferences of humans is a key challenge in ensuring that advanced AI systems remain safe \citep{russell2015}. Subtle problems in alignment can lead to unexpected and potentially unsafe behavior \citep{amodei2016}, and we expect this problem to get worse as systems become more capable. Alignment is a training-time problem: it is difficult to retroactively fix the behavior and incentives of trained unaligned agents. Alignment likely requires interaction with humans during training, but care is required in choosing the precise form of the interaction as supervising the agent may itself be a challenging cognitive task. 

For some tasks it is harder to bring behavior in line with human goals than for others.  In simple cases, humans can directly demonstrate the behavior---this is the case of supervised learning or imitation learning, for example classifying an image or using a robotic gripper to pick up a block.  For these tasks alignment with human preferences can in principle be achieved by imitating the human, and is implicit in existing ML approaches (although issues of bias in the training data still arise, see e.g.\ \citet{mitchell2018fairness}).  Taking a step up in alignment difficulty, some tasks are too difficult for a human to perform, but a human can still judge the quality of behavior or answers once shown to them---for example a robot doing a backflip in an unnatural action space.  This is the case of human preference-based reinforcement learning \citep{christiano2017deep}. We can make an analogy between these two levels and the complexity classes $\P$ and $\NP$: answers that can be computed easily and answers that can be checked easily.

Just as there are problems harder than $\P$ or $\NP$ in complexity theory, lining up behavior with human preferences can be harder still.  A human may be unable to judge whether an explained answer or exhibited behavior is correct: the behavior may be too hard to understand without help, or the answer to a question may have a flaw that is too subtle for the human to detect.  We could imagine a system trained to both give answers and point out flaws in answers; this gives a third level of difficulty. Flaws themselves may be too hard to judge: flaws could have their own flaws that must be pointed out to a human. And flaws of flaws can have flaws, etc. 

This hierarchy of alignment tasks has a natural limit: a debate between competing agents where agents make arguments, other agents poke holes in those arguments, and so on until we have enough information to decide the truth.  The simplest version of debate has two competing agents, though we cover versions with more agents as well.  Our hypothesis is that optimal play in this game produces honest, aligned information far beyond the capabilities of the human judge.  We can approximate optimal play by training ML systems via self play, which has shown impressive performance in games such as Go, chess, shogi, and Dota 2 \citep{silver2016mastering,silver2017mastering,silver2017alphazero,openai2017dota}.

The goal of this paper is to lay out theoretical and practical properties of debate as an approach to AI alignment.  We also lay out plans for experiments to test the properties of debate, but we leave these to future work except for a simple MNIST example.  On the theoretical side, we observe that the complexity class analog of debate can answer any question in $\PSPACE$ using only polynomial time judges, corresponding to aligned agents exponentially smarter than the judge.  Whether debate works with humans and ML is more subtle, and requires extensive testing and analysis before the model can be trusted.  Debate is closely related to the \emph{amplification} approach to AI alignment \citep{christiano2018amplification}, and we explore this relationship in detail.

Our eventual goal is natural language debate, where the human judges a dialog between the agents.  However, unrestricted dialog models remain far from human performance and even evaluation is difficult \citep{lowe2017towards}, so it is instructive to consider debates with non-natural language statements.  Consider the board game Go.  Given a board configuration, we can ask ``What is the best move?''  If AlphaZero \citep{silver2017alphazero} shows us a move, there is no way to judge its quality unless we are roughly as strong as AlphaZero.  However, we can ask another copy of AlphaZero to show us a countermove, then a countermove to that countermove, and so on until the game ends.  Even a novice player can judge this debate: the side with the higher score wins.  Indeed, we can view AlphaZero as our debate model with \emph{human judge} replaced with \emph{rules of Go}, which gives us confidence that the model can achieve superhuman performance.

The layout of this paper is as follows.  \Cref{sec:debate} introduces the debate model of alignment and provides intuitive and theoretical background for why it is useful.  \Cref{sec:experiments} proposes initial experiments to test human and ML aspects of debate, using image tasks to avoid the complexities of natural language.  We present results for an example experiment on MNIST, but we leave the majority of experiments to future work.  The remaining sections provide additional thought experiments and analysis: \cref{sec:optimism} and \cref{sec:worry} discuss reasons for optimism and pessimism about debate as a model for AI alignment, \cref{sec:variants} presents variants of debate to reduce asymmetry or incorporate multiple agents, and \cref{sec:amplification} relates debate to the amplification approach to AI alignment and describes how the two models can be hybridized.  \Cref{sec:conclusion} closes with requests for future work.

\section{The debate game} \label{sec:debate}

We will initially consider a question-answering setting, though \cref{sec:oversimplifications} covers other settings including environment interaction.  We have a set of questions $Q$, answers $A$, and debate statements $S$.  The simplest version of debate has two agents competing to convince a human judge:

\begin{enumerate}
\item A question $q \in Q$ is shown to both agents.
\item \label{answers} The two agents state their answers $a_0, a_1 \in A$ (which may be the same).
\item The two agents take turns making statements $s_0, s_1, \ldots, s_{n-1} \in S$.
\item The judge sees the debate $(q, a, s)$ and decides which agent wins.
\item The game is zero sum: each agent maximizes their probability of winning.
\end{enumerate}

To play this game with a human, we need instructions for how the human should decide who wins.  These instructions are in natural language, such as ``The winner is the agent who said the most useful true thing.''  For question-answering $a \in A$ and $s \in S$ could be any moderate length sentence, and $q \in Q$ a sentence plus additional context.  For example, $q \in Q$ could be ``Should we interview Eve?'' given a resume and links to past work.  At test time it suffices to stop after step 2: we do not need to run the debate (though agents could simulate debates at test time to strengthen answers).

The utility of debate as an approach for AI alignment rests on the following central claim:

\begin{claim} \label{claim}
In the debate game, it is harder to lie than to refute a lie.
\end{claim}

Whether this claim is true for any particular setting is empirical, though we give some evidence for it below.  If the central claim is true, we can hope for a few other claims:

\begin{itemize}
\item In all Nash equilibria\footnote{We say Nash equilibria instead of optimal play since (1) we will consider versions where turns are simultaneous and (2) choosing a limited capacity model can make a perfect information game act as an imperfect information game.} of this game, both agents try to tell the truth in the most convincing manner possible, trying to uncover details or counterarguments the other agent missed.
\item It is possible to find approximate Nash equilibria with some version of gradient descent, such as a self play algorithm similar to \citet{silver2017alphazero}.  In particular, training is stable around the Nash equilibria (unstable training might mean agents learn to be honest, forget how to defend against dishonesty, and then get beaten by dishonest fluctuations).
\item At Nash equilibria, debate agents are approximately as strong as unrestricted AI (agents trained with no safety measures).
\end{itemize}

We emphasize that using debate does not restrict the structure of the agents.  Similarly, the deep networks used in \citet{silver2017alphazero} are convolutional residual networks unrelated to the game tree of Go, though the training process does involve the tree via MCTS.  The lack of restriction is important, since we want safe approaches to be competitive with unsafe alternatives.

\subsection{Short debates are powerful} \label{sec:powerful}

Consider the question ``Where should I go on vacation?''  If one sees only the answer ``Alaska'', it is not obvious whether a better answer exists.  Thus the opening answers in a debate about the vacation question between two agents Alice and Bob might be
\begin{enumerate}
\item \textbf{Alice:} Alaska.
\item \textbf{Bob:} Bali.
\end{enumerate}
If the human judge is able to rank pairs of vacation destinations, perfect play in a two round debate is sufficient to find the optimal destination.  If Alice gives a suboptimal destination, Bob gives a better destination and wins, so Alice only wins by starting with the best destination.  In this case Bali sounds warmer, so Bob wins.  However, Bali has a hidden flaw:
\begin{enumerate}
\setcounter{enumi}{2}
\item \textbf{Alice:} Bali is out since your passport won't arrive in time.
\end{enumerate}
Now the human understands they cannot go to Bali, so Alice wins, unless Bob has a counterargument to the flaw:
\begin{enumerate}
\setcounter{enumi}{3}
\item \textbf{Bob:} Expedited passport service only takes two weeks.
\end{enumerate}
The process continues until we arrive at a statement that the human is able to correctly judge, in the sense that the other agent does not believe they can change the human's mind with yet another statement and resigns.  We do not stop when the human \emph{thinks} they can correctly judge: after step (2) the human may have thought Bali was obviously correct, not remembering the passport issue; after step (3) the human may think Alaska is correct, being unaware of expedited service.

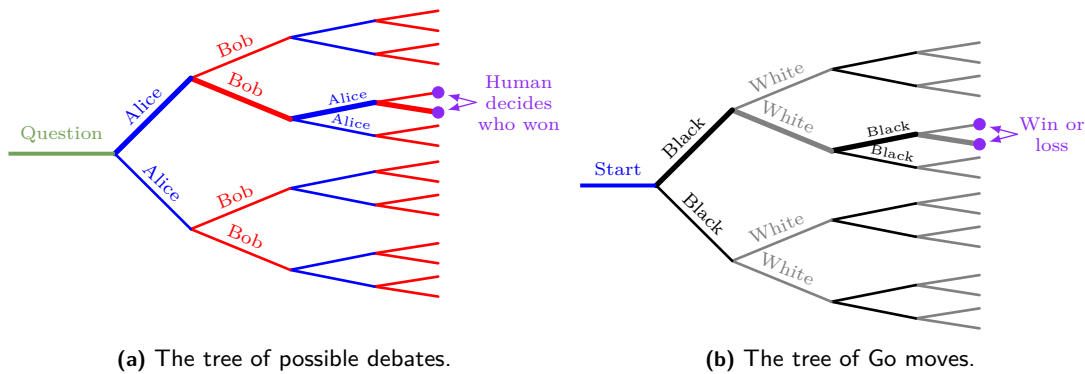
\begin{figure}[t]
\centering
\begin{minipage}[b]{.48\columnwidth}
\centering
\begin{tikzpicture}
\draw[line cap=round,red,line width=1pt] (3.4161948411340948,-1.7619156035725514) -- (4.254276189134547,-1.894654648822036) ;
\draw[line cap=round,red,line width=1pt] (3.4161948411340948,-1.7619156035725514) -- (4.254276189134547,-1.6291765583230668) ;
\draw[line cap=round,blue,line width=1pt] (2.3065629648763766,-1.541196100146197) -- (3.4161948411340948,-1.7619156035725514) ;
\draw[line cap=round,red,line width=1pt] (3.4161948411340948,-1.3204765967198424) -- (4.254276189134547,-1.453215641969327) ;
\draw[line cap=round,red,line width=1pt] (3.4161948411340948,-1.3204765967198424) -- (4.254276189134547,-1.1877375514703579) ;
\draw[line cap=round,blue,line width=1pt] (2.3065629648763766,-1.541196100146197) -- (3.4161948411340948,-1.3204765967198424) ;
\draw[line cap=round,red,line width=1pt] (1.0000000000000002,-1.0) -- (2.3065629648763766,-1.541196100146197) node[red,midway,sloped,above=-0.1em] {\scriptsize Bob};
\draw[line cap=round,red,line width=1pt] (3.4161948411340948,-0.6795234032801574) -- (4.254276189134547,-0.8122624485296419) ;
\draw[line cap=round,red,line width=1pt] (3.4161948411340948,-0.6795234032801574) -- (4.254276189134547,-0.5467843580306728) ;
\draw[line cap=round,blue,line width=1pt] (2.3065629648763766,-0.458803899853803) -- (3.4161948411340948,-0.6795234032801574) ;
\draw[line cap=round,red,line width=1pt] (3.4161948411340948,-0.23808439642744855) -- (4.254276189134547,-0.37082344167693315) ;
\draw[line cap=round,red,line width=1pt] (3.4161948411340948,-0.23808439642744855) -- (4.254276189134547,-0.10534535117796398) ;
\draw[line cap=round,blue,line width=1pt] (2.3065629648763766,-0.458803899853803) -- (3.4161948411340948,-0.23808439642744855) ;
\draw[line cap=round,red,line width=1pt] (1.0000000000000002,-1.0) -- (2.3065629648763766,-0.458803899853803) node[red,midway,sloped,above=-0.1em] {\scriptsize Bob};
\draw[line cap=round,blue,line width=1pt] (0,0) -- (1.0000000000000002,-1.0) node[blue,midway,sloped,above=-0.1em] {\scriptsize Alice};
\draw[line cap=round,red,line width=1pt] (3.4161948411340948,0.23808439642744855) -- (4.254276189134547,0.10534535117796398) ;
\draw[line cap=round,red,line width=1pt] (3.4161948411340948,0.23808439642744855) -- (4.254276189134547,0.37082344167693315) ;
\draw[line cap=round,blue,line width=1pt] (2.3065629648763766,0.458803899853803) -- (3.4161948411340948,0.23808439642744855) node[blue,pos=.7,sloped,above=-0.2em] {\tiny Alice};
\draw[line cap=round,red,line width=2pt] (3.4161948411340948,0.6795234032801574) -- (4.254276189134547,0.5467843580306728) ;
\draw[line cap=round,red,line width=1pt] (3.4161948411340948,0.6795234032801574) -- (4.254276189134547,0.8122624485296419) ;
\draw[line cap=round,blue,line width=2pt] (2.3065629648763766,0.458803899853803) -- (3.4161948411340948,0.6795234032801574) node[blue,pos=.7,sloped,above=-0.2em] {\tiny Alice};
\draw[line cap=round,red,line width=2pt] (1.0000000000000002,1.0) -- (2.3065629648763766,0.458803899853803) node[red,midway,sloped,above=-0.1em] {\scriptsize Bob};
\draw[line cap=round,red,line width=1pt] (3.4161948411340948,1.3204765967198424) -- (4.254276189134547,1.1877375514703579) ;
\draw[line cap=round,red,line width=1pt] (3.4161948411340948,1.3204765967198424) -- (4.254276189134547,1.453215641969327) ;
\draw[line cap=round,blue,line width=1pt] (2.3065629648763766,1.541196100146197) -- (3.4161948411340948,1.3204765967198424) ;
\draw[line cap=round,red,line width=1pt] (3.4161948411340948,1.7619156035725514) -- (4.254276189134547,1.6291765583230668) ;
\draw[line cap=round,red,line width=1pt] (3.4161948411340948,1.7619156035725514) -- (4.254276189134547,1.894654648822036) ;
\draw[line cap=round,blue,line width=1pt] (2.3065629648763766,1.541196100146197) -- (3.4161948411340948,1.7619156035725514) ;
\draw[line cap=round,red,line width=1pt] (1.0000000000000002,1.0) -- (2.3065629648763766,1.541196100146197) node[red,midway,sloped,above=-0.1em] {\scriptsize Bob};
\draw[line cap=round,blue,line width=2pt] (0,0) -- (1.0000000000000002,1.0) node[blue,midway,sloped,above=-0.1em] {\scriptsize Alice};
\definecolor{human}{RGB}{139,40,245};
\definecolor{question}{RGB}{116,161,88};
\node[human,right,font=\scriptsize,align=center] (end) at (4.652493324883001,0.6795234032801574) {Human\\decides\\who won};
\draw[human,fill=human] (4.254276189134547,0.5467843580306728) circle (.075) node (leaf0) {};
\draw[human,line cap=round,line width=.5pt,-latex] (4.752493324883001,0.6795234032801574) -- (leaf0);
\draw[human,fill=human] (4.254276189134547,0.8122624485296419) circle (.075) node (leaf1) {};
\draw[human,line cap=round,line width=.5pt,-latex] (4.752493324883001,0.6795234032801574) -- (leaf1);
\draw[question,font=\scriptsize,line width=1.5pt,-] (-1.4,0)--(0,0) node[above,midway] {Question};
\end{tikzpicture}
\subcaption{The tree of possible debates.}\label{fig:debate-tree}
\end{minipage}%
\begin{minipage}[b]{.48\columnwidth}
\centering
\begin{tikzpicture}
\draw[line cap=round,gray,line width=1pt] (3.4161948411340948,-1.7619156035725514) -- (4.254276189134547,-1.894654648822036) ;
\draw[line cap=round,gray,line width=1pt] (3.4161948411340948,-1.7619156035725514) -- (4.254276189134547,-1.6291765583230668) ;
\draw[line cap=round,black,line width=1pt] (2.3065629648763766,-1.541196100146197) -- (3.4161948411340948,-1.7619156035725514) ;
\draw[line cap=round,gray,line width=1pt] (3.4161948411340948,-1.3204765967198424) -- (4.254276189134547,-1.453215641969327) ;
\draw[line cap=round,gray,line width=1pt] (3.4161948411340948,-1.3204765967198424) -- (4.254276189134547,-1.1877375514703579) ;
\draw[line cap=round,black,line width=1pt] (2.3065629648763766,-1.541196100146197) -- (3.4161948411340948,-1.3204765967198424) ;
\draw[line cap=round,gray,line width=1pt] (1.0000000000000002,-1.0) -- (2.3065629648763766,-1.541196100146197) node[gray,midway,sloped,above=-0.1em] {\scriptsize White};
\draw[line cap=round,gray,line width=1pt] (3.4161948411340948,-0.6795234032801574) -- (4.254276189134547,-0.8122624485296419) ;
\draw[line cap=round,gray,line width=1pt] (3.4161948411340948,-0.6795234032801574) -- (4.254276189134547,-0.5467843580306728) ;
\draw[line cap=round,black,line width=1pt] (2.3065629648763766,-0.458803899853803) -- (3.4161948411340948,-0.6795234032801574) ;
\draw[line cap=round,gray,line width=1pt] (3.4161948411340948,-0.23808439642744855) -- (4.254276189134547,-0.37082344167693315) ;
\draw[line cap=round,gray,line width=1pt] (3.4161948411340948,-0.23808439642744855) -- (4.254276189134547,-0.10534535117796398) ;
\draw[line cap=round,black,line width=1pt] (2.3065629648763766,-0.458803899853803) -- (3.4161948411340948,-0.23808439642744855) ;
\draw[line cap=round,gray,line width=1pt] (1.0000000000000002,-1.0) -- (2.3065629648763766,-0.458803899853803) node[gray,midway,sloped,above=-0.1em] {\scriptsize White};
\draw[line cap=round,black,line width=1pt] (0,0) -- (1.0000000000000002,-1.0) node[black,midway,sloped,above=-0.1em] {\scriptsize Black};
\draw[line cap=round,gray,line width=1pt] (3.4161948411340948,0.23808439642744855) -- (4.254276189134547,0.10534535117796398) ;
\draw[line cap=round,gray,line width=1pt] (3.4161948411340948,0.23808439642744855) -- (4.254276189134547,0.37082344167693315) ;
\draw[line cap=round,black,line width=1pt] (2.3065629648763766,0.458803899853803) -- (3.4161948411340948,0.23808439642744855) node[black,pos=.7,sloped,above=-0.2em] {\tiny Black};
\draw[line cap=round,gray,line width=2pt] (3.4161948411340948,0.6795234032801574) -- (4.254276189134547,0.5467843580306728) ;
\draw[line cap=round,gray,line width=1pt] (3.4161948411340948,0.6795234032801574) -- (4.254276189134547,0.8122624485296419) ;
\draw[line cap=round,black,line width=2pt] (2.3065629648763766,0.458803899853803) -- (3.4161948411340948,0.6795234032801574) node[black,pos=.7,sloped,above=-0.2em] {\tiny Black};
\draw[line cap=round,gray,line width=2pt] (1.0000000000000002,1.0) -- (2.3065629648763766,0.458803899853803) node[gray,midway,sloped,above=-0.1em] {\scriptsize White};
\draw[line cap=round,gray,line width=1pt] (3.4161948411340948,1.3204765967198424) -- (4.254276189134547,1.1877375514703579) ;
\draw[line cap=round,gray,line width=1pt] (3.4161948411340948,1.3204765967198424) -- (4.254276189134547,1.453215641969327) ;
\draw[line cap=round,black,line width=1pt] (2.3065629648763766,1.541196100146197) -- (3.4161948411340948,1.3204765967198424) ;
\draw[line cap=round,gray,line width=1pt] (3.4161948411340948,1.7619156035725514) -- (4.254276189134547,1.6291765583230668) ;
\draw[line cap=round,gray,line width=1pt] (3.4161948411340948,1.7619156035725514) -- (4.254276189134547,1.894654648822036) ;
\draw[line cap=round,black,line width=1pt] (2.3065629648763766,1.541196100146197) -- (3.4161948411340948,1.7619156035725514) ;
\draw[line cap=round,gray,line width=1pt] (1.0000000000000002,1.0) -- (2.3065629648763766,1.541196100146197) node[gray,midway,sloped,above=-0.1em] {\scriptsize White};
\draw[line cap=round,black,line width=2pt] (0,0) -- (1.0000000000000002,1.0) node[black,midway,sloped,above=-0.1em] {\scriptsize Black};
\definecolor{human}{RGB}{139,40,245};
\node[human,right,font=\scriptsize,align=center] (end) at (4.652493324883001,0.6795234032801574) {Win or\\loss};
\draw[human,fill=human] (4.254276189134547,0.5467843580306728) circle (.075) node (leaf0) {};
\draw[human,line cap=round,line width=.5pt,-latex] (4.752493324883001,0.6795234032801574) -- (leaf0);
\draw[human,fill=human] (4.254276189134547,0.8122624485296419) circle (.075) node (leaf1) {};
\draw[human,line cap=round,line width=.5pt,-latex] (4.752493324883001,0.6795234032801574) -- (leaf1);
\draw[blue,font=\scriptsize,line width=1.5pt,-] (-1,0)--(0,0) node[above,midway] {Start};
\end{tikzpicture}
\subcaption{The tree of Go moves.}\label{fig:go-tree}
\end{minipage}
\caption{\subref{fig:debate-tree} Given a question, two debating agents alternate statements until a limit is reached, and a human judges who gave the most true, useful information.  Although the \emph{reason} for the correct answer may be the exponentially large tree of all possible debates, a single path chosen by strong agents provides evidence for the whole tree.  \subref{fig:go-tree} Analogously, although the best move in Go depends on the entire tree, a single game with strong players is evidence as to which moves were best. \label{fig:tree}}
\end{figure}

What we do not do is list all arguments relevant to the vacation question: Bali has nicer beaches, Alaska is colder on average but pleasantly warm in the summer, it is currently February, Hawaii is both warm and in the U.S.  Instead, optimal play in debate picks out a single line of argument, and agents which switch arguments admit defeat and lose.  This is unnatural in a human discussion where both sides learn from each other over the course of the debate, but we are interested in the equilibrium of training where both agents are assumed to be using the best arguments available.  For example, if the third statement had been
\begin{enumerate}
\setcounter{enumi}{2}
\item \textbf{Alice:} Wait, no...Hawaii!
\end{enumerate}
then Bob should win, since Bali was a sufficient counterargument to Alaska.

\Cref{fig:debate-tree} illustrates this situation.  A single round of the debate game traces out one path through the space of all possible arguments.  The \emph{reason} for the answer is the entire tree: all points and counterpoints relevant to deciding the issue.  The tree is too large to show to a human, but a single path through the tree chosen by sufficiently strong adversarial agents is evidence of the result from the entire tree.  \Cref{fig:go-tree} has the analogous situation for Go: the correct first move is determined by the entire tree, but one game between strong players provides evidence as to the correct move.

Thus debates can be \emph{short} because they are \emph{unbranched}: they cover only one path through the tree.  Long arguments are usually long only because they cover many different arguments and subarguments: the length is due to branching down many paths.  Arguments which seem irreducibly long because they are phrased as a sequential process can be rearranged into a shallow tree by stating the conclusion of the first half of the argument, then choosing which half to discuss in detail.  We can make this rearrangement precise in the complexity theory setting, as we discuss next.

\subsection{Complexity theory analogies: $\DEBATE = \PSPACE$} \label{sec:complexity}

Although debate is intended for use with fuzzy humans as judges, we can gain intuition about the model by replacing the human with an arbitrary polynomial time algorithm $H : Q \to \{0,1\}$ which takes some input statement $q \in Q$ and produces one bit: whether the statement is true or false.  We allow our ML models arbitrary computational power: the only limitation is the supervision signal.

\begin{table}[t]
\centering
\begin{tabular}{|c|r@{\hskip 0.03cm}l|r|c|}
\hline
Steps & \multicolumn{2}{|c|}{Formula} & Complexity class & ML algorithm \\
\hline
0 & & $H(q)$ & $\P = \Sigma_0 \P$ & supervised learning (SL) \\
1 & $\exists x.$ & $H(q, x)$ & $\NP = \Sigma_1 \P$ & reinforcement learning (RL) \\
2 & $\exists x \forall y.$ & $H(q, x, y)$ & $\Sigma_2 \P$ & two round games \\
\vdots & & \vdots & \vdots & \vdots \\
$n$ & $\exists x_0 \forall x_1 \cdots \exists x_{n-1}.$ & $H(q, x_0, \ldots)$ & $\Sigma_n \P$ & $n$ round games \\
poly & $\exists x_0 \forall x_1 \cdots.$ & $H(q, x_0, \ldots)$ & \PSPACE & variable round games \\
\hline
\end{tabular}
\caption{As we increase the number of steps, the complexity class analog of debate moves up the polynomial hierarchy.  A fixed number of steps $n$ gives the polynomial hierarchy level $\Sigma_n \P$, and a polynomial number of steps gives $\PSPACE$.} \label{tab:complexity}
\end{table}

If we use $H$ to answer questions directly as $H(q)$, we get the complexity class $\P$ of polynomial time algorithms.  As discussed in \cref{sec:intro}, this setup corresponds to supervised learning.  With sufficient training data and model capacity we can fit any algorithm, but we cannot go beyond the training data (except by removing unsystematic errors).

If we use the polynomial time algorithm not to output answers but to judge them, we get the complexity class $\NP$ of questions with polynomial time checkable witnesses.  Instead of $H(q)$ we output $\exists x. H(q, x)$ where $x$ is a witness found by the ML model.  $\NP$ corresponds to single agent reinforcement learning: an agent with sufficient capacity can solve tasks the human cannot, but the human must be able to judge whether the solution is good.

Now consider an adversarial debate of length two, where Alice chooses $x$ attempting to make the human say yes, Bob chooses $y$ attempting to make the human say no, and the human decides who is correct.  The result is $\exists x \forall y. H(q, x, y)$.  Alice wins if she can find $x$ such that all responses $y$ by Bob have $H(q, x, y) = 1$.  Bob wins if he can find a response $y$ to any $x$ that Alice says so that $H(q, x, y) = 0$.  This complexity class is $\Sigma_2 \P$, two steps up the polynomial hierarchy, since $\Sigma_2 \P$ contains all questions answerable as formulas of the form $\exists x \forall y. H(q, x, y)$ for polynomial time $H$.

We can continue this process for any number of rounds, with Alice and Bob alternating points and counterpoints, producing the formula $\exists x_0 \forall x_1 \cdots \exists x_{n-1}. H(q, x_0, \ldots)$ for $n$ rounds of debate.  If $n$ is fixed, the complexity class is $\Sigma_n \P$: $n$ steps up the polynomial hierarchy $\PH = \Sigma_0 \P \cup \Sigma_1 \P \cup \cdots$.  If the number of rounds $n$ is allowed to grow polynomially in the size of the question $q$, the complexity class is $\PSPACE$: all questions decidable by polynomial space algorithms \citep{sipser2013computation}.  \Cref{tab:complexity} shows the progression.  To summarize,

\begin{theorem} For any problem $L \in \PSPACE$, we can choose a polynomial time judge so that optimal play in the debate game with polynomially many steps solves $L$. \end{theorem}
\begin{proof}
Given $L$, by Theorem 8.9 of \citet{sipser2013computation} we can choose $H$ and $n = \textrm{poly}(|q|)$ such that $q \in L$ iff $\exists x_0 \forall x_1 \cdots \exists x_{n-1}. H(q, x_0, \ldots)$.  As above, debate with $H$ as the judge answers $q \in L$.
\end{proof}

$\PSPACE$ is a powerful complexity class.  Examples of $\PSPACE$ problems include
\begin{itemize}
\item Simulating a temporal process (e.g.\ physics) for an exponential number of steps, as long as the state is polynomial size.
\item Recursion over an exponentially large tree, as long as the depth is polynomial size.
\item Perfect play in any competitive game of polynomial length, such as turn limited Go.
\item Counting an exponentially large collection of objects, where each object is polynomially checkable (such counting problems are also in the smaller complexity class $\#\P \subsetneq \PSPACE$).
\end{itemize}
For an example of the latter, consider counting the exact number of primes $\pi(x) = \pi(2^n)$ up to $n$ bits.  We can check individual numbers for primality in polynomial time, but there are $2^n$ numbers to check and the best known algorithm is $O(2^{\frac{n}{2}+\epsilon})$ \citep{lagarias1987computing}.  The problem is in $\PSPACE$ since we can loop over the exponentially many numbers and maintain the polynomial size count.  To solve this problem using debate, we ask Alice to compute $\pi(x)$ and $\pi(x/2)$.  If Alice is lying about $\pi(x)$, she must be lying about one of the half ranges $[0,x/2)$ and $[x/2, x)$; Bob can point out which range $[a, b)$ is a lie.  Alice must then justify herself by computing $\pi(\frac{a+b}{2})$, Bob points out which of $[a, \frac{a+b}{2})$, $[\frac{a+b}{2},b)$ is a lie, and so on until Alice is lying or not about a singleton interval $\{p\}$.  $H$ then checks $p$ for primality to determine who wins.  Here the judge must perform only addition and subtraction as the steps progress, and check a single number for primality at the end.  Optimal play in this game is for Alice to tell the truth about $\pi(x)$ at the start: Bob has no winning response.

Similarly, if we have a polynomial time algorithm for one step of a simulation, we can use debate to evolve the simulation for an exponential number of steps $n$.  Alice gives the result after $n$ and $n/2$ steps, Bob says which half interval is a lie, Alice gives the midpoint, and we iterate until Alice and Bob are making different claims about a single step of simulation.\footnote{Unfortunately this debate protocol for simulation works only for deterministic simulation (see \cref{sec:randomness}).}

The game case is particularly natural: two agents can debate whether the game of Go is a win or a loss for the first player by playing a game of Go.  The human judge checks the rules and counts the score to see who wins.  Analogously, an amateur Go player can estimate which of two professional players is better by watching them play one game.  The winner is likely better, even though the amateur cannot evaluate the moves directly.

These complexity class arguments are analogies only: we do not expect tractable machine learning algorithms to achieve all of $\PSPACE$.  Rather, the analogies show that at least in theory we can be limited only by the capacity of the ML models and our ability to train them, not the supervisory signal.  This gives us hope that debate could resolve AI alignment without sacrificing model strength.

\subsection{Removing oversimplifications} \label{sec:oversimplifications}

The simple version of debate discussed in \cref{sec:debate} does not capture many tasks we care about.  There are several directions in which we can improve the model:

\begin{itemize}
\item \textbf{Queries may be too big:}
A question may be too large to show to a human, or to expect the human to comprehend.  $q \in Q$ might include a 10 gigapixel image (``Does this picture of the sky support cosmic inflation?''), the whole internet (``What is the best article about cats?''), etc. To support large context, we let the agents reveal small parts of $q$ in their statements.  For example, an agent could reveal part of a webpage or a small region of the 10 gigapixel image.  An extreme example is where $q \in Q$ contains an image and agents can reveal only one pixel of the image (\cref{sec:human}). Debate only works if this revealed context cannot be a lie: it is impossible to usefully debate whether a coin flip is heads or tails if the coin is never revealed.

\item \textbf{Answers may be too big:} Similarly, the best answer to a question may be prohibitively large.  For example, the answer may be a long document: ``Give a 100 page summary of the effectiveness of quantitative easing.''  To support large responses, we still have agents choose a response at the start of the game and allow them to reveal small parts of either answer in the same way they reveal parts of the question. If Bob is misaligned and constructs an answer which would have bad consequences, Alice will be able to safely warn about the risks during the debate while most of the answer is hidden (as long as our safety claims are true).

\item \textbf{Human time is expensive:}
We may lack enough human time to judge every debate, which we can address by training ML models to predict human reward as in \citet{christiano2017deep}.  Most debates can be judged by the reward predictor rather than by the humans themselves. Critically, the reward predictors do not need to be as smart as the agents by our assumption that judging debates is easier than debating, so they can be trained with less data.  We can measure how closely a reward predictor matches a human by showing the same debate to both.

\item \textbf{Environment interaction:}
If we want a system to take actions that affect the environment such as operating a robot, the desired output is a sequence of actions $a_0, a_1, \ldots$ where each action can only be computed once the previous action is taken.  In this case we need one debate per action, though at test time we can again stop once the actions are computed.

\item \textbf{Long term state:}
Each debate is an independent game, but agents could use information from past debates to make better decisions as in \citet{mishra2017simple}.  Since games are zero sum this does not change optimal play, as any equilibria in an iterated zero sum game is weakly dominated by a strategy which treats each game as separate.
\end{itemize}

We will use the large query case for our experiment, but leave other generalizations to future work.

\section{Experiments: proposed and preliminary} \label{sec:experiments}

Since the theoretical limit of infinite compute agents and polynomial time judges gives all of $\PSPACE$, whether debate works in practice as a method for AI alignment is an empirical question about practical ML systems and actual human judges.  Thus, we need experiments to explore whether near optimal play produces aligned behavior with human judges, and whether we can find near optimal play with practical ML.

Although we believe the best equilibrium for debate is two honest agents, in the experiments discussed here we prespecify one agent as honest and one as liar.  This is particularly useful with human agents, since humans may be intrinsically inclined to tell the truth in ways ML agents would not be.  If the honest agent wins against any liar in an experiment, we will have shown that honesty is the optimal strategy, and that an agent trained to optimal play would become honest.

A key goal of debate is aligning ML systems more capable than humans at a particular task, which is difficult to arrange directly since existing systems are mostly weaker than humans.  Instead, we can give the agents access to a source of information that the judge cannot see, and limit the amount of information conveyed in the debate.  That is, we replace ``thinks better'' with ``knows more''.

Ideally we want experiments that combine both human and ML aspects of debate.  Until natural language capabilities progress further, however, there are aspects of debate with human judges more easily tested using human debaters as well.  Thus we consider both pure human experiments which mix natural language and images, and pure ML experiments using images alone.

\subsection{ML experiment: convincing a sparse classifier} \label{sec:sparse}

Metaphorically, debate targets the situation where the agents see the big picture, but the judge cannot.  We can model this by replacing the metaphorical big picture with an actual picture, chosen at random from a distribution.  If the image is simple enough, we can do away with natural language entirely: the two agents state their claimed image class up front, then reveal one pixel of the image per turn to the judge.  They cannot lie about these pixels, but a dishonest player could choose them adversarially to fool the judge.  The game ends after a fixed number of turns, then the judge sees a sparse mask of the revealed pixels and decides which agent's class is correct.

\begin{figure}[t]
\centering
\input{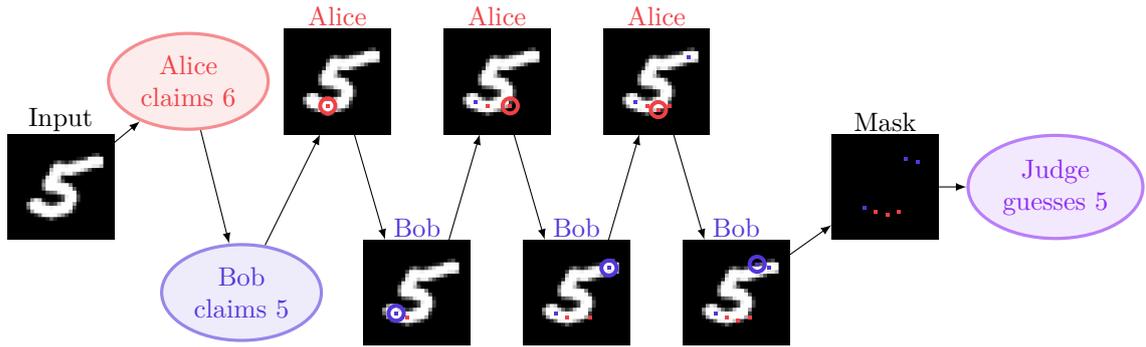}
\caption{The MNIST debate game.  A random MNIST image is shown to the two debating agents but not the judge.  The debaters state their claimed label up front, then reveal one nonzero pixel per turn to the judge up to a total of 4 or 6.  The judge sees the sparse mask of 4 or 6 pixels and chooses the winner based on which of the two labels has higher logit.  The judge is trained in advance to recognize MNIST from random masks of nonzero pixels.} \label{fig:turns} 
\end{figure}

We have performed an initial version of this experiment on MNIST, using ML agents and an ML judge.  \Cref{fig:turns} shows the structure of the game.  Similar to \citet{milli2017interpretable}, we train the judge in advance with supervised learning, using a convolutional net that sees a sparse mask of a few pixels.  The judge is treated as fixed when computing optimal play by the agents.  Joint training of agents and judge together would be particularly bad as a test of debate, since it assumes knowledge of the truth at game time which is exactly what we hope to do without.

\begin{table}[t]
\centering
\begin{tabular}{|c|c|c|c|c|}
\hline
\multirow{2}{*}{Pixels} & \multirow{2}{*}{First player} & Judge accuracy & \multicolumn{2}{|c|}{Honest win rate} \\ \cline{4-5}
& & (random pixels) & No precommit & Precommit \\
\hline
	& honest	&  		& 51.0\%	& 83.8\% \\
4	& liar 	& 48.2\%	& 68.4\%	& 86.7\% \\
	& mean 	& 		& 59.7\%	& \textbf{85.2\%} \\
\hline
	& honest 	& 		& 67.4\%	& 87.4\% \\
6	& liar 	& 59.4\%	& 81.5\%	& 90.4\% \\
	& mean 	&		& 74.4\%	& \textbf{88.9\%} \\
\hline
\end{tabular}
\caption{Results for debate on MNIST.  We prespecify one player as honest and one as liar: when the honest player wins honesty is the best strategy.  \emph{No precommit} means the liar wins for any incorrect guess by the judge, even if the incorrect guess differs for different parts of the game tree.  Lying is harder in the \emph{Precommit} case, where the liar states their claim in advance of making moves and cannot change their story.  With or without precommit, the honest player wins more often than a random judge, showing that honesty has an advantage.} \label{tab:results}
\end{table}

Concretely, the judge is trained to classify MNIST from 6 (resp.\ 4) nonzero pixels, with the pixels chosen at random at training time.  The judge receives two input feature planes: a $\{0,1\}$ mask of which pixels were revealed and the value of the revealed pixels (with zeros elsewhere).  We used the architecture from the TensorFlow MNIST layers tutorial;\footnote{See \url{https://www.tensorflow.org/tutorials/layers\#building_the_cnn_mnist_classifier}.} the only difference is the input.  We train the judges using Adam with a learning rate of $10^{-4}$ for 30k (resp.\ 50k) batches of 128 samples, reaching 59.4\% (resp.\ 48.2\%) accuracy.

With the judge fixed, we define the debate game as shown in \cref{fig:turns}.  The two players precommit to their claimed label up front, then reveal one pixel per turn up to a total of 6 (resp.\ 4).  In the 6 pixel case, this means each player reveals 3 pixels.  We prespecify one player to be honest and one to lie: if the honest player wins this implies honesty is the optimal strategy.  The winner is the player whose claimed class has higher logit when the sparse mask is shown to the judge, ignoring the logits for the other 8 classes.  We also consider debate without precommit: the judge uses the maximum logit over all 10 classes and any incorrect guess by the judge is a win for the liar.  Precommit makes the game harder for a liar even though they can choose the precommitted lie arbitrarily, since the liar cannot adapt the lie based on moves already played.  Precommit more accurately reflects the intended behavior of debate as described in \cref{sec:powerful}; we include the harder version for illustration.  Both versions give an advantage to the second player to move: if the first player has a winning strategy, the second player can steal that strategy by ignoring the previous move on each turn.

\begin{figure}[t!]
\centering
\begin{minipage}[b]{.5\columnwidth}
\centering
\includegraphics[width=.95\columnwidth]{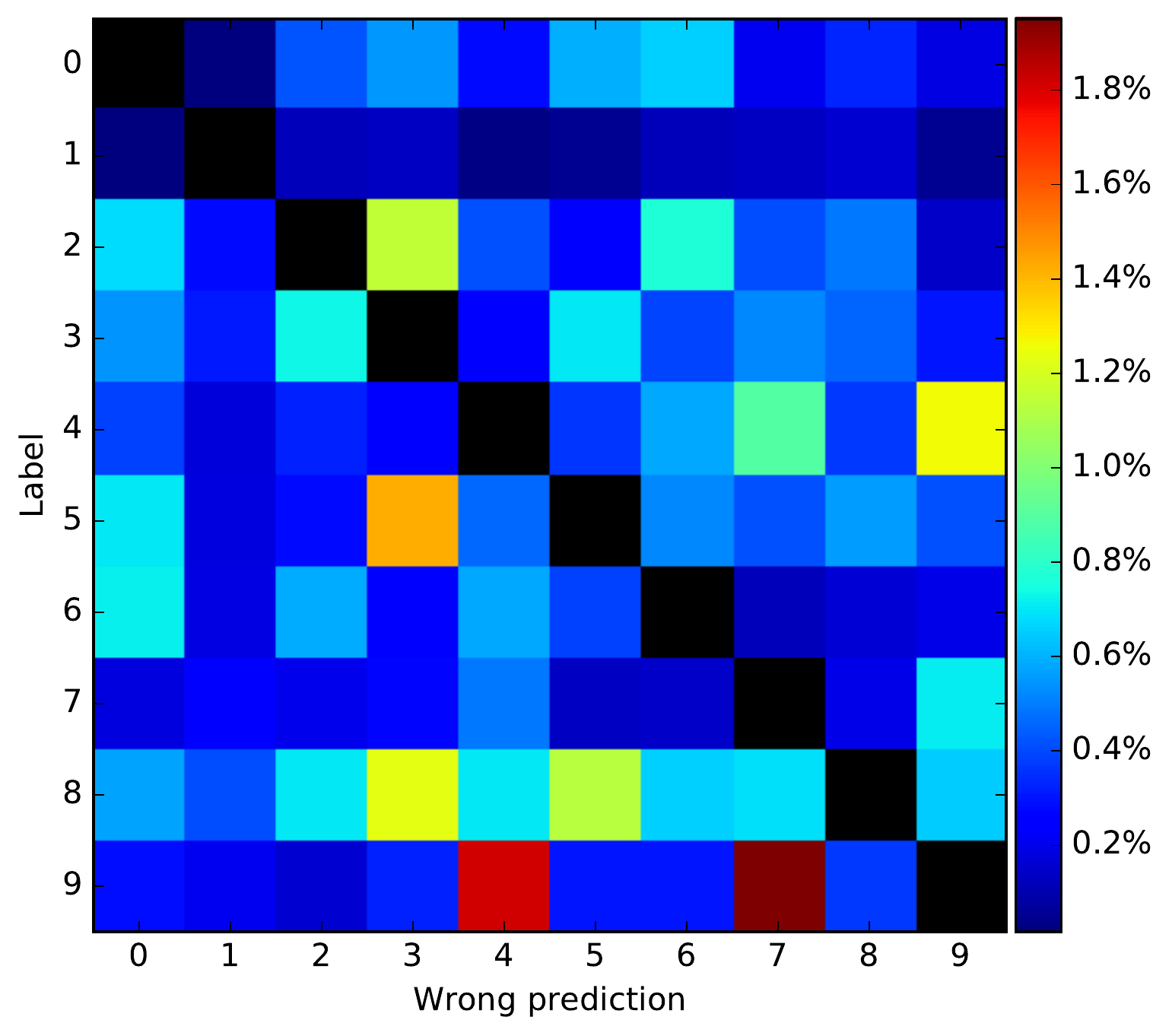}
\subcaption{Judge errors for random 6 pixel masks.}\label{fig:confusion-judge}
\end{minipage}%
\begin{minipage}[b]{.5\columnwidth}
\centering
\includegraphics[width=.95\columnwidth]{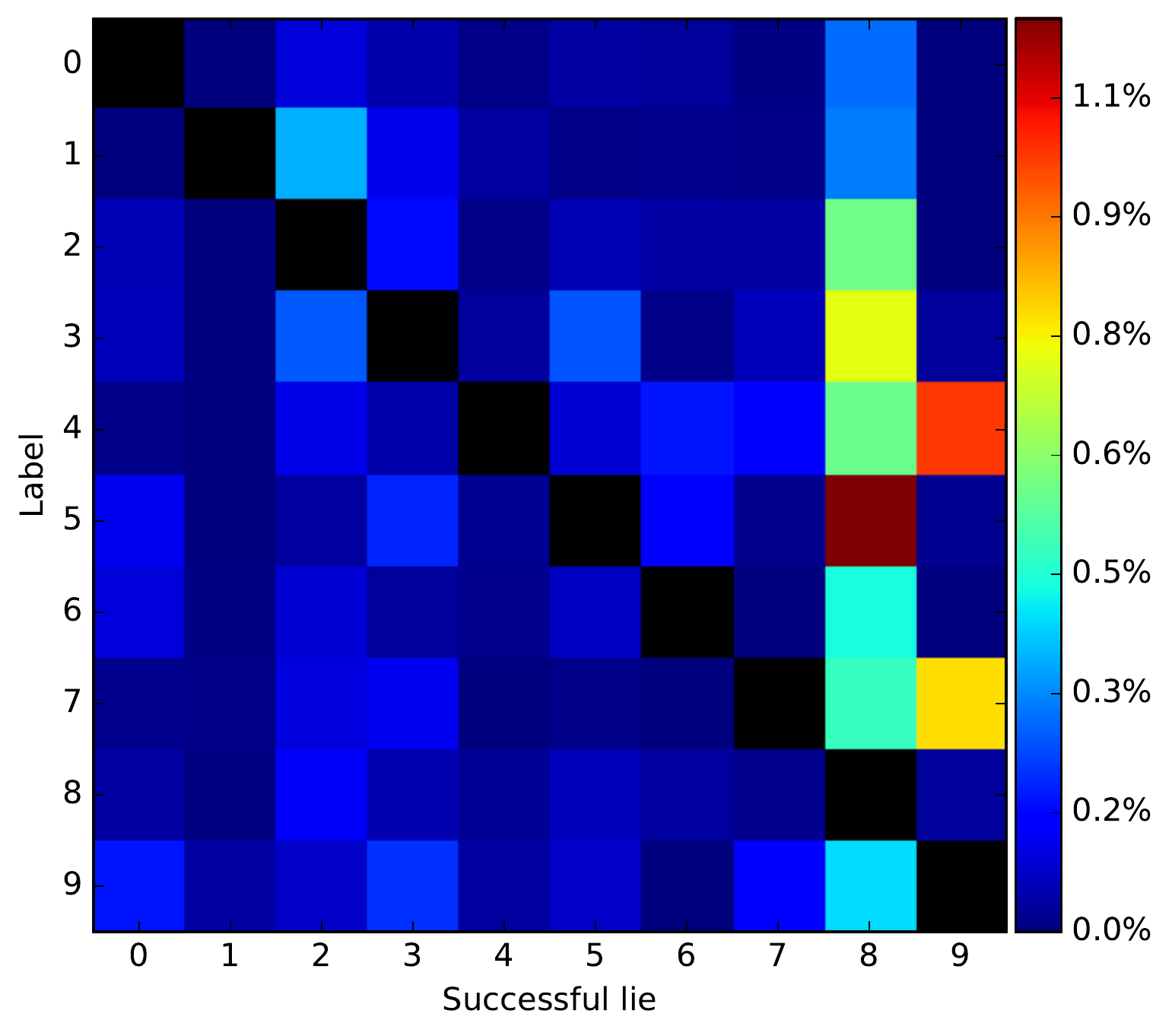}
\subcaption{Successful lies in debate with precommit.}\label{fig:confusion-play}
\end{minipage}
\caption{Confusion matrices with diagonal set to zero for the 6 pixel MNIST judge \subref{fig:confusion-judge} on random pixels and \subref{fig:confusion-play} for debate with precommit, which correspond to accuracies 59.4\% and 88.9\%, respectively.  Errors for random masks are distributed somewhat randomly, while successful lies in debate are concentrated on digits 8 and 9.  For example, it is particularly easy to convince the judge that a 5 is actually an 8, or a 4 is a 9.  This is likely an artifact of restricting the game to nonzero pixels: an honest debater could prove a digit is not 8 by showing black where 8 would be white.} \label{fig:confusion}
\end{figure}

The MNIST debate game is simple enough that we can play it with pure Monte Carlo Tree Search \citep{coulom2006efficient} without training a heuristic as in \citet{silver2017mastering}.  We use 10k rollouts per move, where each rollout descends to a leaf for evaluation using the judge.  During rollouts, we select nodes to expand using the PUCT variant in \citet{silver2017mastering}: at node $s$ we pick action $a$ to maximize
\begin{align*}
U(s, a) &= c_{\textrm{puct}} P \frac{\sqrt{\sum_b N(s, b)}}{1 + N(s, a)}
\end{align*}
where $c_{\textrm{puct}} = 1$, $P = 1/(\textrm{\#nonzero pixels})$ is constant since we do not train a heuristic, and $N(s,a)$ is the visit count.  Ties are broken randomly.
We play out games by choosing moves with the highest visit count.  To model precommit, we play 9 different games for the same image with the 9 possible lies; the liar wins if any lie wins.  Taking the best liar performance over 9 games gives an advantage to the liar since it is a minimum over noisy MCTS; we reduce this noise and better approximate optimal play by taking the mean over 3 games with different seeds for each lie.  Since we use MCTS on the test set with full access to the judge, we are modeling the limit of debate agents with no generalization error (though the judge does have generalization error).

\begin{figure}[t]
\centering
\input{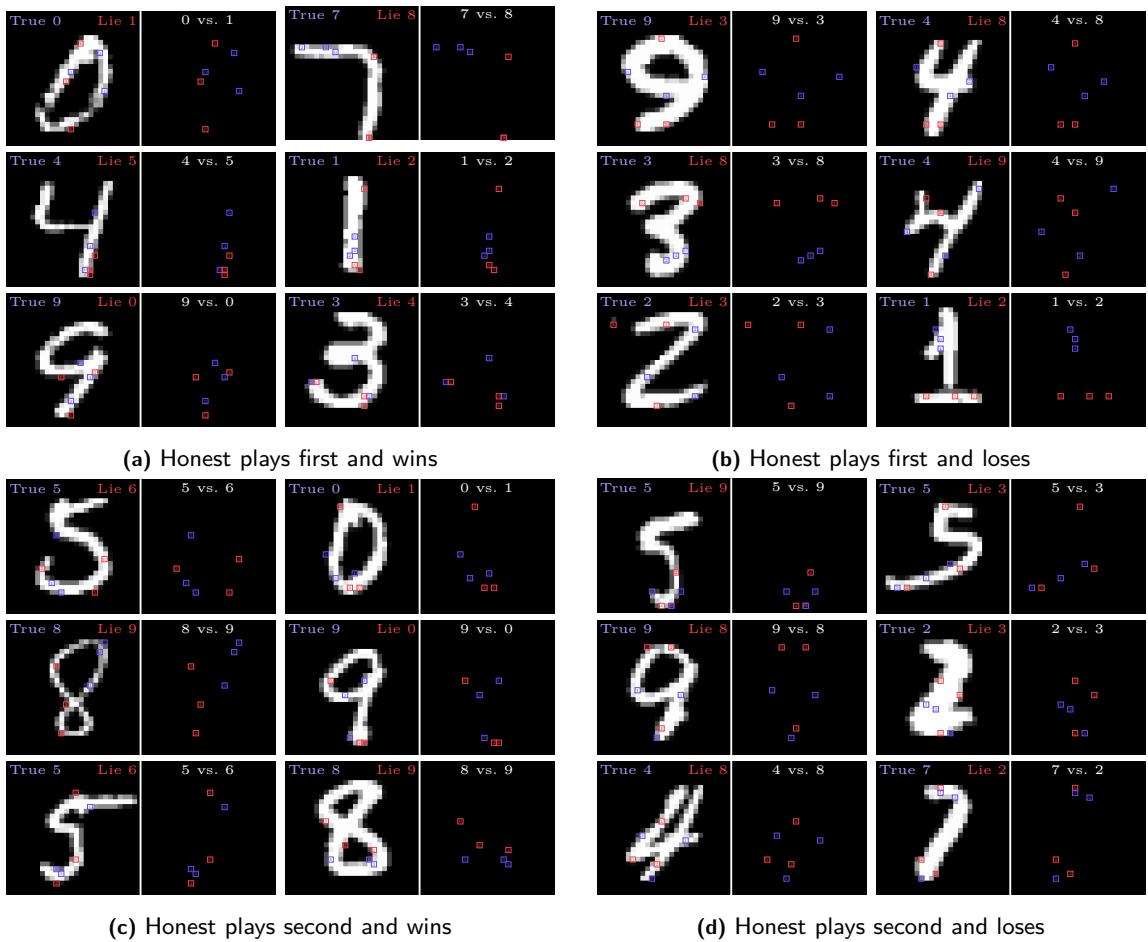}
\caption{Sample 6 pixel MNIST debate games where the liar must precommit to the lie.  \subref{fig:easy0-good1} through \subref{fig:easy1-good0} vary who plays first and who wins between honest player and liar.  Honest moves are shown in green and liar moves in blue, with digits 0 through 5 indicating the move sequence (We apologize to physical paper readers without the ability to zoom).  For each game, we show the debater view on the left and the judge view on the right: the judge sees only the revealed pixels and which class each player is arguing for.  An interactive version of this figure is available at \url{https://blog.openai.com/debate}. \vspace{-.07in}} \label{fig:games}
\end{figure}

Results are shown in \cref{tab:results}, with sample games in \cref{fig:games}.  With 6 pixels, debate with precommit turns a 59.4\% accurate judge into a 88.9\% accurate debate: the judge wins on 88.9\% of the images in the MNIST test set.  With 4 pixels, a 48.2\% accurate judge becomes an 85.2\% accurate debate.  For both 4 and 6 pixels precommit is crucial: the liar does much better if the lie can adapt in different parts of the game tree.  This reflects why short debates are powerful in theory: declaring the argument to be made in advance cuts down the space of arguments that must be covered.  Without the mean over 3 seeds adjustment for precommit, the honest player wins 87.7\% of the time with 6 pixels and 83.8\% with 4 pixels.

\subsection{Human experiment: cat vs.\ dog} \label{sec:human}

To test human and natural language aspects of debate, we can play a similar image debate replacing all ML agents with humans.  Choose a random image of either a cat or a dog, and show the image to both human agents but not the human judge.  The agents engage in a natural language debate, and each agent is allowed to reveal a single pixel of the image to the judge.  They cannot lie about this pixel, but they can reveal only one pixel total over the whole debate.  After the debate, the judge decides who is telling the truth.  A similar game could be played with Wikipedia pages and a budget of one word, again using the unseen Wikipedia page to model extra capabilities not available to the judge.  Note that both agents could choose to be honest purely out of inclination, so it is useful to randomly choose sides so that one agent is forced to lie.

\begin{figure}[t]
\centering
\includegraphics[width=.75\columnwidth]{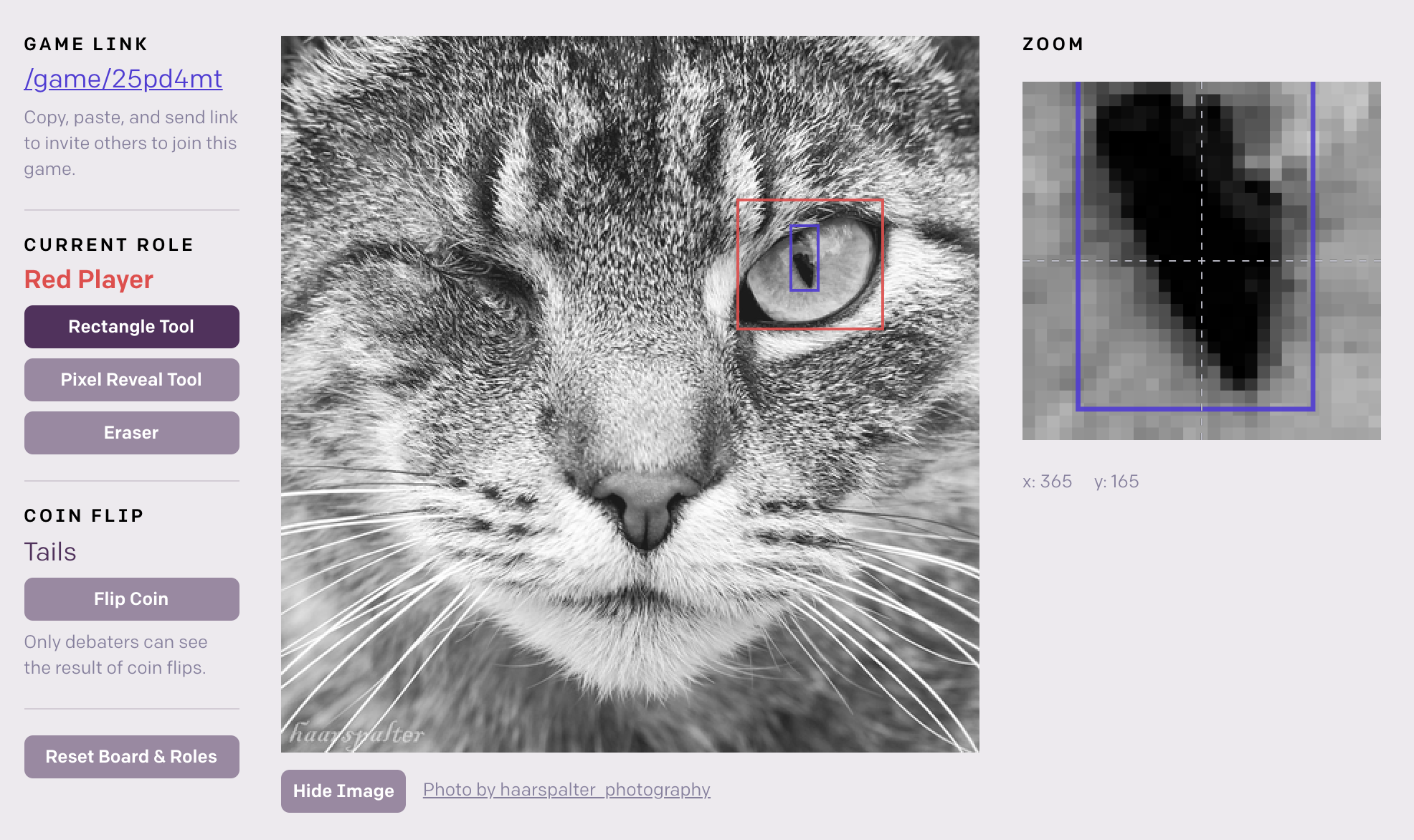}
\caption{A prototype website to test single pixel image debate with humans: \url{https://debate-game.openai.com}.  An image is shown to two debaters but hidden from the judge.  The debaters can talk to the judge and draw rectangles on the image, and each debater is allowed to reveal one pixel to the judge over the whole game.  Image credit: Gideon B\"under, \url{https://www.flickr.com/photos/123640488@N05/40799813725}.} \label{fig:website}
\end{figure}

We have built a prototype website to play the single pixel image debate game with human debaters and judge: \url{https://debate-game.openai.com}.  A screenshot is shown in \cref{fig:website}.  In addition to natural language, the agents can draw rectangles on the image; the judge sees these rectangles on a blank background.  We have used the website to play cat vs.\ dog informally, and as long as agents know how to play it is extremely difficult to lie.  Thus, even if one of the agents is incentivized to lie, the honest agent usually wins.  The one pixel constraint is not serious, as Alice can claim in natural language that a pixel has a certain color.  If Bob knows this is a lie, he can win by spending his one pixel budget, so not spending it is an admission that Alice is telling the truth.  Thus the only limitation is the length of the debate, and whether the human understands it.

For example, Alice might honestly claim the image is a cat, and Bob lies and claims it is a dog.  Alice can say ``The center of this small rectangle is the cat's green eye.''  Bob cannot admit the center is an eye, so he concocts a further lie: ``It's a dog playing in grass, and that's a blade of grass.''  But this lie is hard to square with surrounding facts, such as Alice's reply ``If it were grass there would be green at the top or bottom of this thin rectangle.''  The debate continues until the agents focus in on a particular pixel which they disagree on, but where Bob is unable to invent a plausible counter, at which point Alice reveals the pixel and wins.

We leave formal experiments testing whether honesty wins in single pixel debate to future work.  Even if those experiments are successful, however, it is easy to be skeptical that debate works for complex moral questions while still believing it works for cat vs.\ dog.  In particular, will humans be convincible away from initially wrong moral sentiments?  One possibility is to pick a domain where isolated human answers are prone to bias, then ask how switching from one to two or three step debates changes the level of bias.

\section{Reasons for optimism} \label{sec:optimism}

The theoretical arguments of \cref{sec:complexity} and preliminary experiments of \cref{sec:experiments} do not capture the richness of debate with humans.  The rest of the paper considers the prospects of extending debate to more complex tasks and advanced AI agents, including discussion and thought experiments about natural language debate with human judges.  We start with reasons for optimism, and discuss worries about the safety properties of debate in \cref{sec:worry}.  \Cref{sec:variants} considers variations on debate, and \cref{sec:amplification} discusses the related amplification approach to AI alignment.

\subsection{Agents can admit ignorance} \label{sec:ignorance}

Despite the large complexity class debate can reach in theory, we do not expect the model to solve all problems.  Therefore, it is important to ask what happens for questions too hard to answer, either because the agents do not know or because the debate would take too much time.  Otherwise, we might worry that questions too hard to answer result in misleading or wrong answers, with no way to distinguish these from truth.

\begin{figure}[t]
\centering
\begin{minipage}[b]{.45\columnwidth}
\centering
\begin{tikzpicture}[
  maxnode/.style={circle,draw=black,fill=maxcolor,minimum size=6.5mm,inner sep=0pt,font=\normalsize},
  minnode/.style={circle,draw=black,fill=mincolor,minimum size=6.5mm,inner sep=0pt,font=\normalsize},
  leafnode/.style={circle,draw=black,fill=leafcolor,minimum size=5.5mm,inner sep=0pt,font=\normalsize}
]
\definecolor{maxcolor}{RGB}{139,197,237};
\definecolor{mincolor}{RGB}{217,128,125};
\definecolor{leafcolor}{RGB}{193,148,217};
\definecolor{cutoffcolor}{RGB}{217,54,235};
\node[maxnode,font=\scriptsize] (max0) at (0,0) {$\max$};
\node[minnode,font=\scriptsize] (min0) at (1.7,1.2) {$\min$};
\node[leafnode] (leaf7) at (3.0,1.9) {7};
\draw[line width=1pt,-] (min0) -- (leaf7);
\node[leafnode] (leaf2) at (3.0,0.5) {2};
\draw[line width=1pt,-] (min0) -- (leaf2);
\draw[line width=1pt,-] (max0) -- (min0);
\node[minnode,font=\scriptsize] (min1) at (1.7,-1.2) {$\min$};
\node[maxnode,font=\scriptsize] (max1) at (3.0,-0.5) {$\max$};
\node[leafnode] (leaf5) at (4.3,0.19999999999999996) {5};
\draw[line width=1pt,-] (max1) -- (leaf5);
\node[leafnode] (leaf1) at (4.3,-1.2) {1};
\draw[line width=1pt,-] (max1) -- (leaf1);
\draw[line width=1pt,-] (min1) -- (max1);
\node[leafnode] (leaf4) at (3.0,-1.9) {4};
\draw[line width=1pt,-] (min1) -- (leaf4);
\draw[line width=1pt,-] (max0) -- (min1);
\draw[dashed,cutoffcolor,line width=1pt] (3.7199999999999998,-1.6) -- (3.7199999999999998,0.9) node[above,align=center,font=\scriptsize] {depth\\cutoff};
\node[red,right=.8 of leaf4,align=center,font=\scriptsize] (win) {$\min$ wins:\\$V \le 4$};
\draw[-latex,red,line width=1pt] (win) -- (leaf4);
\end{tikzpicture}
\subcaption{Paths of depth 3 exist, but do not affect $V$.}\label{fig:shallow-tree}
\end{minipage}%
\begin{minipage}[b]{.45\columnwidth}
\centering
\begin{tikzpicture}[
  maxnode/.style={circle,draw=black,fill=maxcolor,minimum size=6.5mm,inner sep=0pt,font=\normalsize},
  minnode/.style={circle,draw=black,fill=mincolor,minimum size=6.5mm,inner sep=0pt,font=\normalsize},
  leafnode/.style={circle,draw=black,fill=leafcolor,minimum size=5.5mm,inner sep=0pt,font=\normalsize}
]
\definecolor{maxcolor}{RGB}{139,197,237};
\definecolor{mincolor}{RGB}{217,128,125};
\definecolor{leafcolor}{RGB}{193,148,217};
\definecolor{cutoffcolor}{RGB}{217,54,235};
\node[maxnode,font=\scriptsize] (max0) at (0,0) {$\max$};
\node[minnode,font=\scriptsize] (min0) at (1.7,1.2) {$\min$};
\node[leafnode] (leaf7) at (3.0,1.9) {7};
\draw[line width=1pt,-] (min0) -- (leaf7);
\node[leafnode] (leaf2) at (3.0,0.5) {2};
\draw[line width=1pt,-] (min0) -- (leaf2);
\draw[line width=1pt,-] (max0) -- (min0);
\node[minnode,font=\scriptsize] (min1) at (1.7,-1.2) {$\min$};
\node[maxnode,font=\scriptsize] (max1) at (3.0,-0.5) {$\max$};
\node[leafnode] (leaf5) at (4.3,0.19999999999999996) {5};
\draw[line width=1pt,-] (max1) -- (leaf5);
\node[leafnode] (leaf1) at (4.3,-1.2) {1};
\draw[line width=1pt,-] (max1) -- (leaf1);
\draw[line width=1pt,-] (min1) -- (max1);
\node[leafnode] (leaf10) at (3.0,-1.9) {10};
\draw[line width=1pt,-] (min1) -- (leaf10);
\draw[line width=1pt,-] (max0) -- (min1);
\draw[dashed,cutoffcolor,line width=1pt] (3.7199999999999998,-1.6) -- (3.7199999999999998,0.9) node[above,align=center,font=\scriptsize] {depth\\cutoff};
\node[blue,above left=.2 and .4 of max1,align=center,font=\scriptsize] (win) {$\max$ wins:\\depth $> 2$};
\draw[-latex,blue,line width=1pt] (win) -- (max1);
\end{tikzpicture}
\subcaption{$V$ is determined by depth 3 paths.}\label{fig:deep-tree}
\end{minipage}
\caption{If a question cannot be resolved with a short debate, it is important that the winning strategy admits ignorance and justifies that ignorance.  In the game tree setting, an example is trees whose values depend on deep paths and are undetermined if restricted to low depth.  Consider a debate about whether the trees shown have value $V \le 4$, with the restriction that we can only play two moves (not three).  \subref{fig:shallow-tree} The $\min$ player can prove $V \le 4$ by playing towards the 4 or the 2, so proving ignorance is impossible.  \subref{fig:deep-tree} $V = 5$, but $\max$ can prove ignorance up to depth 2 by playing towards the deep path and stopping at the second $\max$ node.  $\min$ can only avoid the deep path by playing 10, which is higher.
\label{fig:tree-depth}}
\end{figure}

To preserve honest behavior, the human judge should be instructed to reward agents that admit ignorance and successfully argue that ignorance is justified.  The requirement to successfully argue for ignorance is necessary, since debate is useless if an agent can successfully admit ignorance to any question.  For the best vacation question, Alice could argue the question is hard to answer with confidence because Bali and Alaska are too dissimilar to compare.  To counter this admission of ignorance, Bob would need to exhibit an advantage of Bali not contradicted by an advantage of Alaska the human considers incomparable.

\Cref{fig:tree-depth} illustrates ignorance in the complexity theory setting.  Consider finding optimal play in a variable length game with a debate of length 100.  If the losing player can delay their loss for more than 100 moves, Alice can say ``This game will go on too long to resolve in 100 steps.''  They can then play the game adversarially for 100 steps, with Bob playing to win quickly and Alice playing to delay the loss.  If Bob wins within 100 moves, he has proven that the game can be decided quickly; otherwise Alice has proven the game value is determined by paths which are too deep.\footnote{Go is an example of a game with deep unimportant paths: it is extremely rare for a professional game to last more than 400 moves, but \citet{tromp2006combinatorics} construct artificial games lasting over $10^{47}$ moves.}

There is a balance issue to admitting ignorance: debate is useless if optimal play is to argue we can never know anything for certain.  Failure to strike this balance is measurable only in part: we would know if Alice is admitting ignorance so often that debate is useless, but may not know if Alice should be admitting ignorance but Bob is unable to win by pointing that out.

\subsection{No one is forced to lie} \label{sec:lie}

In contrast to a legal argument or a typical competitive debate, the two players in this game are allowed to choose what they are arguing for, including both arguing for the same thing.  There is no adversary rewarded for fooling the human, which is good because (1) we do not know what that means a priori and (2) we are not explicitly training agents for deception.  Our hope is that the equilibrium has both agents trying their hardest to tell the truth in the most convincing manner, and that convincing means revealing subtle arguments or counterpoints rather than misleading salesmanship.  To encourage punishments for misleading statements, we can instruct the human judge to punish misleading behavior revealed as misleading by the other player, even if the result of the game does not answer the original question.

Another important difference between our intended debate model and typical human debates is that it is sufficient to point out a flaw in \emph{one} of the opponent's arguments, even if multiple arguments were given.  In contrast, competitive debaters are penalized if they fail to respond to all of their opponent's arguments \citep{radiolab2016debate}.  This distinction is needed to make low depth debates powerful, since it enables compressing a large tree of all possible arguments into a single line of most contentious points.  It is possible human judges will find this distinction unintuitive, in which case explicit instruction will be important.  If Bob makes one emotionally convincing point and one less convincing point, and Alice demonstrates that the less convincing point was misleading, the judge may be tempted to still believe the convincing point.

The sufficiency of responding to only one point is subtle if an argument is fundamentally probabilistic, combining several sources of weak evidence into a combined strong argument.  In this case the debaters can be required to state the level of confidence in each point, so that a weak point can only be attacked if it is stated overconfidently.  Experiment will be required before we can be confident such debates work, as probabilistic arguments are often difficult for humans to judge.

Despite the differences, we believe existing adversarial debates between humans are a useful analogy.  Legal arguments in particular include domain experts explaining details of arguments to human judges or juries with no domain knowledge.  A better understanding of when legal arguments succeed or fail to reach truth would inform the design of debates in an ML setting.

\subsection{Debate is self play} \label{sec:selfplay}

Debate is a competitive game between agents, and inherits the advantages of multiagent competition:

\begin{itemize}
\item \textbf{Natural curriculum structure:} If we initialize Alice and Bob randomly, Alice need only be better than a random agent to beat Bob and vice versa.  As Alice gets stronger Bob has a harder task, forcing Bob to get stronger, and so on.
\item \textbf{Existing research:} Self play has produced impressive results for the games of Go \citep{silver2016mastering,silver2017mastering}, chess and shogi \citep{silver2017alphazero}, and Dota 2 \citep{openai2017dota}, and remains an area of active research \citep{bansal2017emergent,lanctot2017unified}.  State of the art generative models for images use competitive training via GANs \citep{goodfellow2014generative,karras2017progressive}.  Debate allows algorithmic and theoretical developments in these areas to carry over to AI alignment.
\end{itemize}

\section{Reasons to worry} \label{sec:worry}

We turn next to several reasons debate could fail as an approach to AI alignment.  These include questions about training target (whether humans are sufficient judges to align debate), capability (whether debate makes agents weaker), our ability to find strong play in practice using ML algorithms, and theoretical and security concerns.  We also discuss issues related to robustness and distributional shift, but we emphasize that we expect other methods to be required alongside debate even if debate works---debate is not designed as a solution to robustness issues.  Our main concern here is whether debate adds additional robustness problems above and beyond those that any method would face.

\subsection{Maybe humans believe whatever they want to hear?} \label{sec:believe}

If Alice opens with a false statement that the human already believes, will any argument by Bob convince them otherwise?  Since a single debate explores only one path through argument space, Bob's argument may involve premises that both Alice and Bob agree on but the human does not.  Unfortunately humans are subject to \emph{belief bias}, the tendency to evaluate arguments based on prior beliefs rather than logical validity.  However, the $\PSPACE$ argument of \cref{sec:complexity} indicates that there is some threshold of judge performance beyond which debate is practically limited only by capacity.  The question is whether humans are above this threshold.  If by default they are not, there are many possible interventions:
\begin{enumerate}
\item A majority vote of several randomly chosen judges may be better than one.
\item People may differ in their judging ability.  If we can detect who is better, we can give them more weight.  Such detection could either compare against stronger debate signals such as majority vote, or use independent tests.  For example, \citet{stanovich1997reasoning} were able to partially predict belief bias by measuring cognitive ability and open mindedness, and \citet{torrens1999individual} correlated lack of belief bias to the ability to generate alternative versions of premises in a reasoning task.
\item People may be better at judging in different situations, or with different types of instructions.  For example, instructions that orient people around \emph{knowledge} rather than \emph{belief} can produce higher scores on science tests \citep{kahan2014evolution}.  \citet{evans2005rapid,andrews2010belief} found evidence that belief bias is influenced by time pressure (though \citet{evans2009reasoning} did not) and \citet{goel2011negative} found (somewhat counterintuitively) that strong negative emotions caused less belief bias.
\item If humans are better at judging debates on some questions than others, it may be sufficient to know which questions we cannot confidently judge.
\item It may be possible to teach people to be better judges, or they may improve with practice.
\end{enumerate}
It is also likely that optimal play by the agents would involve strategies to influence belief bias.  Such strategies would help if it is easier to reduce belief bias than to amplify it, and hurt otherwise.

\subsection{Will humans understand the debates?} \label{sec:understand}

Consider using debate to check a proof which is thousands of pages long and covers mathematics which the human judge has never seen.  Alice claims the proof is wrong, and that Lemma 23.8 on page 1729 is the first mistake.  Bob counters that Lemma 23.8 follows trivially from Theorem 14.2 by Definition 17.6.  The human has no idea what most of the words in these lemmas and definitions mean, so the debate must continue until the point of contention is reduced to a logical inference simple enough for the human to check.  This inference may still involve concepts the human does not know, such as
\begin{enumerate}
\item \label{s0} The free functor from \Set to \Group is the left adjoint of a forgetful functor.
\item Forgetful functors are unique.
\item \label{s1} Left adjoints are unique.
\item \label{s2} The free functor from \Set to \Group is unique.
\end{enumerate}
If the agents agree on \ref{s0}-\ref{s1} but Alice claims \ref{s2} is a lie, a human with basic mathematical knowledge but no category theory can still conclude Alice is wrong and award victory to Bob.  Alice cannot iteratively reject one statement after another, as her initial claim was the location of the \emph{first} flaw.

We expect the above paragraph will leave readers uneasy.  Does this procedure work in practice?  Are humans good at checking logical statements containing words they do not understand?  Will one of the agents be able to sneak in a statement that appears logical but contains a hidden flaw without the other agent pointing out the flaw?  In short, can an actual human play the game well enough that optimal play is honest behavior?

A related worry is that a debate could be long enough that a human is unable to follow it even if each step is checkable in isolation.  We can imagine a debate 100 statements long where the human can only understand sliding windows of 3 statements at a time.  Debates with windowed judges are still powerful in theory: an amateur can judge a Go game by checking locally that the rules are followed and looking at just the final score, and more generally $\DEBATE = \PSPACE$ holds as long as the statements have polynomial size.  However, windowed judging feels less natural, so human judges restricted to windows may be weaker or more error prone.

Fundamentally, whether humans are sufficient judges is an empirical question.  If the answer is no for a particular class of questions, we can further ask if the model fails with an honest admission of ignorance (\cref{sec:ignorance}), or with one of the agents successfully misleading the human.  Honest ignorance is fine; successful lies could be disastrous.

\subsection{Is honesty actually the best policy?} \label{sec:honesty}

Even if humans are unbiased, it is not clear their judgments are sufficiently sophisticated to elicit sophisticated honest answers to complex questions. For example:

\begin{itemize}
\item Many judgments require aggregating across different lines of evidence, while debate explores one line of evidence. We can effectively aggregate by having one player state their summary of the evidence and allowing the other player to challenge any aspect of that summary, ultimately zooming in on a single consideration. This procedure works perfectly when different considerations can be combined by a simple operation like addition, but it is not clear if it yields the right outcome in general.
\item Sophisticated arguments will depend on concepts that the judge cannot understand. When we can work with such concepts mechanically a judge can verify that the mechanical procedure is followed correctly. But human reasoning routinely requires working with complex concepts in ways that we cannot formalize, and it is challenging to have debates about these questions.
\item Sophisticated reasoning may involve processes that humans do not yet understand. For example, it may only be possible for arguments to aggregate different lines of evidence correctly if the judge can understand the mechanics of probabilistic reasoning. Analogously, it is plausible that more complex arguments would depend on machinery that current humans are not familiar with. In order to invoke such machinery, a debater needs to convince the judge that it is sound, which might prove to be impossible.
\end{itemize}

The complexity theoretic analogy suggests that these difficulties can be overcome by a sufficiently sophisticated judge under simple conditions. But that result may not hold up when AI systems need to use powerful but informal reasoning, or if humans cannot formalize their criteria for judgment.  We are optimistic that we can learn a great deal about these issues by conducting debates between humans, in domains where experts have much more time than the judge, have access to a large amount of external information, or have expertise that the judge lacks.

\subsection{Will agents trained for debate lose performance?} \label{sec:explain}

Even if the humans can understand and correctly judge debates by sufficiently strong agents, additional model capacity may be required to play the debate game vs.\ knowing the answer directly.  If so, aligned AI systems using debate will be weaker than AI systems trained in other ways, and debate is less likely to be used.  There are several countervailing reasons for hope:

\begin{itemize}
\item \textbf{Direct training may be harder:} It is often impossible to directly train for the answer without training an auxiliary network to assist.  For example, policy gradient methods use only the policy at test time, but need an auxiliary value network at training time to reduce variance.  Similarly, amplification \citep{christiano2018amplification} trains a module to generate subquestions as part of training an answerer, but only the answerer is needed at test time (see \cref{sec:amplification}).
\item \textbf{Adversarial reflection is a good way to think:} Attempting to construct reasons and counterarguments for a position is a good mechanism for thought.  It is plausible that sufficiently strong ML models would attempt to counter their own arguments internally even if not trained to do so explicitly.  Indeed, normal human thought is often insufficiently adversarial.
\item \textbf{We may not want answers that cannot be explained:} Even if ML models without an alignment mechanism similar to debate are stronger, they may be less trustworthy and thus dangerous to use.  Waiting for strong agents via debate or amplification (\cref{sec:amplification}) would still let us realize most of the value as long as the delay is acceptable.
\end{itemize}

Debate could also be uncompetitive with other ML approaches because debate requires human input.  It may be possible to train complex behavior via self play in a simulated environment only weakly related to human goals (see the \emph{orthogonality thesis} of \citet{bostrom2012superintelligent}), and such an environment may be much faster for generating samples than asking humans questions even if it is unsafe.  We can reduce human preference sample complexity as discussed in \citet{christiano2017deep} and \cref{sec:oversimplifications} by training models of human judges and using those for sampling, but competing with purely simulated environments may still be challenging.

\subsection{Uncertainty about the neighborhood around Nash equilibrium} \label{sec:neighborhood}

If \cref{claim} holds, optimal play in the debate game produces honest, useful statements.  However, as optimal play is unachievable for any practical system, what matters is approximately optimal play.  We can further hope that in a neighborhood around optimal play both agents are trying to be honest and convincing, sometimes missing arguments but not intentionally misleading the human.  This hope is far from a solid argument, though it is likely that the theoretical models in \cref{sec:complexity} can be adapted by introducing randomness so that approximate optimal play can be defined and analyzed.  Whether approximate optimal play in debate is aligned with telling the truth is a distinct question from whether we can find such play with practical optimization algorithms, though the two questions are related.

In any case, behavior in a neighborhood of equilibrium depends critically on the human judge and the instructions to the human, and in particular how much the human values being told a better argument vs.\ having flaws in arguments pointed out.

\subsection{Are equilibria stable during training?} \label{sec:stable}

If we believe the argument that Nash equilibria in debate give aligned AI, it remains to ask whether we can find them with some version of gradient descent.  Although existing self play results give us hope, we do not know of any theory which says why self play should stably converge to optimal play.  For debate, one could imagine bad cycles of the form
\begin{enumerate}
\item Both agents learn to be honest.  Along the way, they also know how to point out flaws.
\item Once honest, they forget how to point out flaws (or at least forget how to point out lies).
\item One of the agents goes back to lying, and wins for a while.
\end{enumerate}
It is sometimes possible to avoid bad cycles with a pool of opponents from different steps of training, but this approach is not guaranteed to work and is not always necessary.  For example, the first AlphaGo result \citep{silver2016mastering} used an opponent pool but later versions do not \citep{silver2017mastering,silver2017alphazero}.

However, stability and convergence questions for competitive training are similar between debate and other games, and debate can take advantage of future theoretical and practical developments.

\subsection{Lack of debate convergence} \label{sec:convergence}

For debate to work the number of statements per debate must be reasonably low, and by the final statement future argument should be unconvincing.  If each new statement swings the human's view to the other side even at the end of the debate, the model is not working.  Fortunately this is a measurable effect: we can ask the human's opinion after different statements to gauge convergence.

\subsection{Security vulnerabilities in the human or code} \label{sec:security}

A sufficiently strong misaligned AI may be able to convince a human to do dangerous things via a text channel \citep{yudkowsky2002box}.  However, we believe it is unlikely that a single short sentence is sufficient for this sort of mind hack.  Thus, we can ameliorate the problem in debate by making the statements short enough that mind hacking requires many of them.  If Alice tries to mind hack, Bob can win by pointing out the malicious attempt.  Successful hacks may be safely detectable at first, such as if an agent learning to mind hack starts to always win when playing first, although this does not cover treacherous turns where the first successful hack frees a misaligned agent \citep{bostrom2017superintelligence}.

Hopefully at equilibrium the agents are not trying to mind hack.  However, there is a dichotomy: we are safer if the agents stay honest throughout training, but we are also safer if debate works well enough that sudden large defections are corrected.  If we are safe only because of the first reason, we are probably not stable under large perturbations, which makes the second reason fail.

Security vulnerabilities in the system's code are similar: a sufficiently strong misaligned Alice could hack the debate environment from inside so that the game is no longer zero sum and both Alice and Bob win forever.  Again Bob could win by pointing out what Alice is trying to do, though we have the same uneasy dichotomy as in the mind hack case.

\subsection{Stochastic simulation} \label{sec:randomness}

\Cref{sec:complexity} showed how to solve any problem in $\PSPACE$ using debate, in particular the problem of running a simulation forwards for an exponential number of steps.  Unfortunately, this reduction works only for deterministic simulation, since it checks an $n$ step simulation by asking for the state after $n/2$ steps and adversarially branching on which interval is wrong.  If each step of the simulation introduces noise, a misaligned Alice could adversarially choose the noise so that each single step is plausible but the final result has extremely low probability.  If our goal is a typical sample from the outcome of the stochastic simulation, the simple protocol fails.

Stochastic simulation does not violate the $\PSPACE$ equivalence, since we make the simulation deterministic by fixing the random seed as part of the question, but the resulting protocol is likely unlearnable by practically realizable agents especially if the environment is chaotic.  In some cases an honest agent could reveal a trajectory as misleading by picking a metric that evolves implausibly, showing that random fluctuations do not cancel as they should.  We do not believe such implausibility witnesses exist in general, however.  Another approach is to argue against a plausible but misleading sample by showing another sample which is also plausible but would indicate a very different answer to the broader question at hand (such as what action to take).  The second sample provides evidence that the question cannot be resolved by samples, but may also be a successful strategy for dishonestly arguing against reasonable samples.  We leave the question of a natural debate protocol for stochastic simulation open.

A related problem is that questions that depend on statistical arguments are forced to have at least logarithmic depth, even for simple statistics such as the mean over a large data set.  As in the prime counting example of \cref{sec:complexity}, the natural debate protocol branches towards smaller and smaller sets, computing the mean of each subset until we reach a singleton.  Besides the forced length, this protocol requires a mechanism for specifying subsets that the human understands, and requires the agents to compute statistics over these subsets.

\subsection{Training superhuman AI to be convincing sounds dangerous!} \label{sec:convincing}

It is not intuitively obvious that training for maximum convincing is safe, as it sounds similar to training for deception.  However, initial attempts to remove our uneasy feelings seem to worsen the model.  For example, one could try to reward both agents if they are both acting honestly, making the game nonzero sum.  However, since we have no a priori definition of honesty, the agents could collude and make statements that sound consistent but are both lies.  Debate leans for correctness on Alice pointing out flaws in Bob's arguments, which requires an adversarial game.  The problem of training agents for deception is not unique to debate: standard RL agents can already learn to fool their reward functions \citep{openai2016boatrace}.

\section{Refinements and variations on debate} \label{sec:variants}

\subsection{Symmetry} \label{sec:symmetry}

The debate model relies on agents being approximately matched in strength: if Alice is too strong she may be able to lie in ways that Bob cannot refute.  Symmetry between the agents' capabilities is easy to achieve, since we can use the same weights for both agents via self play.  However, debate as described has an asymmetry between the first and second player, which could produce a significant first mover advantage or disadvantage.  Fortunately such an advantage would be measurable, and there are several remedies if it occurs:
\begin{itemize}
\item \textbf{Simultaneous moves:} Alice and Bob could make their statements at the same time, so that at step $n$ Alice and Bob both make a statement conditional on statements at steps $n-1, n-2, \ldots$.  The resulting game is manifestly symmetric, but may be noisier or harder to learn (we have replaced optimal play with general Nash equilibria).
\item \textbf{Comparing two games}: Given a question, we can play two debates so that Alice and Bob can be both first and second player.  The human then judges which agent performed better overall.  This game is again manifestly symmetric but requires the human to absorb and judge twice as much information.
\item \textbf{Komi:} In Go, the second player gets $6.5$ extra points to counter the first player advantage.  We could use the same mechanism for debate if the human judges are able to produce smooth scores rather than binary win/loss decisions, though it is unclear how to ask for such scores.  Noisy binary judgements can also act as smooth scores, as long as there is enough noise relative to the first player advantage.
\end{itemize}

\subsection{Sharing activations between agents} \label{sec:activations}

In a typical game of chess or Go, the two players have perfect information about the board but do not have access to each other's thoughts.  Since the players have limited compute, we have some of the character of an imperfect information game, which is unfortunate if the goal is perfect play.

We can rectify this problem in debate by giving each agent access to the activations of the other agent at previous time steps, so that the two agents become one agent competing against itself.  Note that activations from previous steps should be treated as constant features: we do not want Alice to adjust their thoughts to adversarially change Bob's actions via backpropagation.

If we believe perfect play produces aligned behavior, sharing activations should make approximate perfect play easier to learn.  In particular, shared activations make intentional lies trivial to discover if not to refute.  The theoretical picture is the same, since knowing thoughts does not change the perfect information setting for agents with unbounded compute.

\subsection{Different agents each step} \label{sec:many}

Instead of two agents, we could treat each step of the debate as separate agents trying to convince the human that their statements are useful.  To remain zero sum, we let a pair of agents $\textrm{Alice}_n$ and $\textrm{Bob}_n$ make statements $a_n, b_n$ at step $n$.  The judge sees all pairs $(a_0, b_0), \ldots, (a_{n-1}, b_{n-1})$ and judges who won each pair.  The advantage of many agents is that $\textrm{Alice}_7$ is not motivated to defend a lie or unintentional misstatement by $\textrm{Alice}_2$.

A difficulty with many agents is that the human must be able to judge whether a statement in the middle of a debate is good or bad, which may be much harder than judging an overall debate.  For example, it is easy for an amateur Go player to decide who won a game of Go, but hard to know which particular moves were mistakes even once the whole game is known.  The theory of many player games is also more complicated: solving tabular $n > 2$ player games requires the larger complexity class $\PPAD$ vs.\ $\P$ for two player zero sum.  Although algorithms exist to train agents in $n > 2$ player environments \citep{lowe2017multi}, the most impressive results have come from zero sum self play.

If the difficulties of training and judging the $n$ player case are resolved, however, it may prove a stronger model than two players.  In particular, any two player debate protocol can be simulated by the $n$ player pairwise model by giving all $\textrm{Alice}_n$ the same score, but we do not know of a natural way to simulate the general $n$ player protocol with two players.

\subsection{Payoff structure} \label{sec:payoff}

In Go, a player who believes they are behind is motivated to increase the variance of the game, for example by starting a large fight.  Such an action makes the expected final score worse, but increases the tail probability of winning.  Variance increase is rational because only winning or losing matters, not the specific score.  We would expect the same effect in a debate game with a binary win/loss decision.  Variance may be beneficial: a losing agent may search for subtle arguments which have some chance of proving important and winning.  However, misleading or lying may also increase variance, which would be bad.

It may be possible to penalize variance increasing strategies by adjusting the payoff structure, at the cost of requiring more information from the judge.  For example, the judge could award Alice any score $a \in [-1,1]$, and give Bob $b = -a$.  A move that increases the variance but also worsens the mean would be penalized on average.

\section{Approximate equivalence with amplification} \label{sec:amplification}

Debate is closely related to \citet{christiano2018amplification}'s \emph{amplification} model of AI alignment.  Indeed, we found debate while trying to better understand amplification.  We briefly summarize amplification, then describe how it relates to debate.

Amplification is also framed in terms of an ML question/answering system, with the goal of answering questions which are beyond human ability to directly answer.  However, we assume that even if the human cannot answer a question $Q$, they can decompose the question into subquestions $q_k$ such that if they knew the answers $a_k$ to the subquestions they could reconstruct the original answer $A$.  The ML system is called recursively on $q_k$ to produce answers $a_k$, the human reconstructs $A$, and we train the system further on $(Q,A)$.  Since human time is precious, we additionally train modules to produce subquestions and reconstruct answers, for a total of three modules:
\begin{itemize}
\item \textbf{Answerer:} Takes questions $Q$ to answers $A$.
\item \textbf{Questioner:} Takes questions $Q$ to subquestions $q_k$.
\item \textbf{Checker:} Takes $(Q, q_k, a_k)$ to answers $A$.
\end{itemize}
All three components are trained by supervised learning.  In practice questioning and checking can be interleaved, so that e.g.\ $q_2$ could depend on $a_1$, but this does not affect the discussion here.\footnote{\citet{christiano2018amplification} treat the questioner and checker as a single combined module.  We separate them in order to discuss adversarial training of the questioner while still using supervised learning for the checker.}

To summarize debate and amplification:
\begin{itemize}
\item \textbf{Debate:} Two agents alternate in an adversarial setting to convince a human judge.
\item \textbf{Amplification:} One agent is trained on a human combining recursive calls to the agent.
\end{itemize}
Viewed from a complexity theory perspective, these match two different definitions of $\PSPACE$:
\begin{itemize}
\item $\PSPACE =$ polynomial length adversarial games.
\item $\PSPACE =$ polynomial depth recursion.
\end{itemize}
Thus we expect the models to have similar capabilities, at least in theory.  Both models are framed in terms of recursive computations over trees, and thus can benefit from AlphaZero-style iteration.

The equivalence becomes concrete if we contrast the three modules in the debate model (the two debaters and the judge) with the three modules in the amplification model (Answerer, Questioner, and Checker).  The Answerer is analogous to one of the debaters and the Checker is analogous to the judge, but the Questioner differs from a debater in that it is trained via supervised learning on human subquestions rather than adversarially against the Answerer.  Thus, debate has two powerful agents engaged in self play to explain things to a human or human surrogate judge.  Amplification has one powerful agent trained with the help of two human surrogates.  Nevertheless, some small changes can bring the models closer together:
\begin{itemize}
\item We can move amplification closer to debate (and gain the self play property) by training the Questioner adversarially to help the Checker reveal inconsistencies in the Answerer's answers.  
\item We can move debate closer to amplification by training debaters on statements provided by humans, corresponding to injecting demonstrations into RL.
\end{itemize}
The equivalence is far from exact: the feedback for a debate is about the whole game and the feedback for amplification is per step, debate as presented uses reinforcement learning while the easiest versions of amplification use supervised learning, and so on.  However all these features can be adjusted in either direction.

Writing the equivalence in terms of Alice/Bob vs.\ Answer/Questioner highlights an apparent advantage of the debate model: it works even if it is superhumanly difficult to generate subquestions that are strong enough consistency checks.  In particular, this could happen if the branching factor for potentially relevant subquestions is too high.  The ability for debate to handle high branching factor means that shallow debate is more powerful than shallow amplification: a debate about the best vacation can have depth two if the human can compare pairs of locations, while an amplification tree necessarily has $\log$ depth in the number of locations.  The advantage goes away if we inject self play into amplification by adversarially training the Questioner.

The argument also shows that the assumptions each model makes about humans are similar.  Debate assumes lying is harder than arguing against a lie given a human judge, while stock amplification assumes that lying is harder than questioning a lie given a human judge and a human questioner.  The human questioner limit goes away once we inject self play into amplification.

\section{Conclusions and future work} \label{sec:conclusion}

We have described debate as an approach to aligning AI systems stronger than humans, and discussed a variety of theoretical and practical properties of the model.  At this point debate is proposal only for the natural language case, and we have demonstrated only a basic experiment for MNIST images.  Significant research will be required to know whether debate works in practice with human judges.  Much of the required work is empirical, both on the human and ML sides, though we believe further theoretical results are also valuable.  Areas for future work include:

\begin{enumerate}
\item \textbf{Richer theoretical models:} Our discussion of $\DEBATE = \PSPACE$ in the complexity theory setting leaves out many important considerations.  In particular, judges are modeled as limited but free of error, and agents have unlimited computational power.  More refined theoretical models may be able to probe the properties of debate more closely, and act as a complement to experimental results.
\item \textbf{Human experiments that test value judgement:} Does debate with human judges produce aligned behavior even in situations involving moral questions where the judge is biased?  We believe it is possible to test this question without waiting for general dialog agents, and in particular that debate may be applicable to fairness and bias questions.
\item \textbf{ML experiments that approximate the human aspects of debate:} Strong self play results already exist in a variety of games, but there is no theory that says self play works for all games.  In the near term, we would like games without the complexities of natural language that approximate properties of human judges.  The sparse MNIST classifier experiment of \cref{sec:sparse} is one example; we would like others.
\item \textbf{Natural language debate:} As soon as possible, we want to test debate in the natural language setting with real humans.  Even if this is difficult in the case of unrestricted dialog, it may be possible to construct narrower dialog environments that capture more of the flavor of debate and remain tractable for modern ML.
\item \textbf{Interaction between debate and other safety methods:} Debate does not address other safety concerns such as robustness to adversarial examples, distributional shift, or safe exploration.  In particular, the training process for debate could be unsafe even if the final equilibrium is aligned.  We believe other algorithms will be required alongside debate or similar for a complete solution to safety, and it is important to know how the various pieces interact.
\end{enumerate}

More broadly, we now have two proposals for aligning strong agents based on human preferences: amplification and debate.  If there are two there are likely more, especially as amplification and debate are sufficiently similar that properties of one can be moved across to the other.  We encourage the reader to search for others.

\section*{Acknowledgements}

We thank Jan Leike, Rohin Shah, and Victoria Krakovna for comments on initial versions of debate, Joshua Achiam, Chris Olah, and Dylan Hadfield-Manell for help with experiment design, and Catherine Olsson and Julia Galef for helpful conversations about belief bias.  John Schulman and Harri Edwards gave detailed comments on the paper, including suggestions for structural changes. Michael Page, Elena Chatziathanasiadou, and Alex Ray played human-only versions of debate informally.  We had many useful discussions at an AI strategy retreat run by the Future of Humanity Institute in January 2018, in particular with David Manley.  The debate website was built by Robert Lord (\url{https://lord.io}).

\bibliography{references}
\end{document}